\documentclass[preprint,12pt,authoryear]{elsarticle}

\let\today\relax
\makeatletter
\def\ps@pprintTitle{%
    \let\@oddhead\@empty
    \let\@evenhead\@empty
    \def\@oddfoot{\footnotesize\itshape
         {Submitted preprint} \hfill\today}%
    \let\@evenfoot\@oddfoot
    }
\makeatother

\usepackage{amssymb}
\usepackage{amsthm}

\usepackage{lineno}
\usepackage{url}

\usepackage{amsmath}
\usepackage{algorithm}
\usepackage{algorithmic}
\usepackage{xcolor}
\usepackage{booktabs}       
\usepackage{multirow}

\newtheorem{Theorem}{Theorem}

\newtheorem{remark}{Remark}

\begin{document}

\begin{frontmatter}



\title{Physics-Informed Neural Networks with Hard Linear Equality Constraints}


\author[purduecheme]{Hao Chen}

\author[purduecheme]{Gonzalo E. Constante Flores}

\author[purduecheme]{Can Li\corref{cor}}
\ead{canli@purdue.edu}
\cortext[cor]{Corresponding author at: Davidson School of Chemical Engineering, Purdue University, USA.}

\affiliation[purduecheme]{organization={Davidson School of Chemical Engineering, Purdue University},
            addressline={480 W. Stadium Ave}, 
            city={West Lafayette},
            postcode={47907}, 
            state={IN},
            country={USA}}
            
\begin{abstract}
Surrogate modeling is used to replace computationally expensive simulations. Neural networks have been widely applied as surrogate models that enable efficient evaluations over complex physical systems. Despite this, neural networks are data-driven models and devoid of any physics. The incorporation of physics into neural networks can improve generalization and data efficiency. The physics-informed neural network (PINN) is an approach to leverage known physical constraints present in the data, but it cannot strictly satisfy them in the predictions. This work proposes a novel physics-informed neural network, KKT-hPINN, which rigorously guarantees hard linear equality constraints through projection layers derived from KKT conditions. Numerical experiments on Aspen models of a continuous stirred-tank reactor (CSTR) unit, an extractive distillation subsystem, and a chemical plant demonstrate that this model can further enhance the prediction accuracy.
\end{abstract}




\end{frontmatter}



\section{Introduction}
\label{sec:introduction}

With the development of computational capability and capacity, computer simulations have attained widespread use across various fields, facilitating high-fidelity modeling and analysis of complex systems or processes. These simulations are governed by a large number of algebraic equations, ordinary or partial differential equations (ODEs or PDEs), or a combination of both differential and algebraic equations (DAEs). These equations are derived from fundamental principles and mechanistic laws, such as the physical laws in thermodynamics and transport phenomena. High-fidelity models with these equations can serve as digital representations of the physical systems in the real world. However, the physically accurate representation is accompanied by a heightened mathematical complexity that elevates the computational expense of simulation. This impedes the use of high-fidelity physical models especially in applications where it is essential to simulate a system repeatedly in a timely manner. 

To efficiently generate simulation outputs, data-driven approaches have sought to substitute a high-fidelity physical model with a surrogate model
\citep{Misener2023, Bhosekar2018, Bradley2022, Williams2021},
A surrogate model stands for a reduced-order model that aims for a computationally efficient approximation at the cost of a certain level of accuracy. This approach provides a more practical means of inferring a system's responses under a great variety of conditions. Among them, feed-forward neural networks have been proven to be capable of approximating any continuous function theoretically. As a universal approximator, it has demonstrated effectiveness in capturing complex and nonlinear relationships in high-dimensional space, with remarkable advancements in computer vision and natural language processing \citep{LeCun2015}. 

Hence, neural networks have received substantial attention as surrogate models, which facilitate their deployment in time-sensitive and large-scale applications \citep{Bradley2022, Dias2020, Kim2020, Mohammadi2022}. For instance, incorporating simulators or underlying equations into optimization can become exceedingly expensive, whereas neural networks offer a more efficient alternative. On the one hand, black-box optimization that enables the use of simulators without an explicit algebraic form is not practical when dealing with a high degree of freedom. Even with access to high-fidelity physical models, optimization remains challenging despite advances in state-of-the-art solvers. The algebraic equations often give rise to nonconvex mixed-integer nonlinear programming (MINLP), and the DAEs introduce optimization problems with ODE or PDE constraints, which can be intractable. On the other hand, neural networks with ReLU activation layers can be integrated as surrogate models within mixed-integer linear programming (MILP) problems to replace the nonlinear functions \citep{Tsay2021}.

On the flip side, the drawback of neural networks is the lack of physical interpretability. This limitation originates from two factors. First, the intricate architecture of neural networks is composed of neurons hierarchically arranged in multiple layers. It is challenging for humans to understand the role of each neuron due to the complexity of hierarchical representations. Second, as a data-driven method, neural networks are black-box models whose parameters are not physics-informed. Training a neural network is equivalent to minimizing an objective function, often referred to as a loss function, over these parameters. While this process updates parameters by evaluating the error between model prediction and ground truth, it does not take into account any physics. Therefore, the model prediction cannot conform to any physical laws. The lack of physical interpretability can hinder their application in high-stakes decision-making problems, where undetected errors or misunderstandings within the model could potentially lead to catastrophic losses.

To mitigate this drawback, the physics-informed neural network (PINN) was developed to leverage the physical constraints that data are subject to in particular application domains. In the prevalent PINN method, a penalty term is incorporated into the loss function to account for the violation of constraints \citep{Karniadakis2021}. 
This can be considered a multi-objective optimization problem, where the goal is to minimize the prediction errors and violation of physics simultaneously at the cost of each other. However, even with the sacrifice of some prediction accuracy, the PINN method still cannot guarantee the exact satisfaction of physical constraints. In other words, PINN only strikes a balance between approximating the ground truth and favoring the first principles during the training process. Hence, this is also named ``soft constraint'', as they are not strictly satisfied.

The inclusion of soft constraints in PINN still cannot avoid physically inconsistent results, rendering it unsuitable for certain applications. For instance, PINN cannot strictly enforce conservation laws when trained to represent a chemical unit. In practical scenarios, it is necessary to partition a large system into distinct units, each of which is then replaced with surrogate models \citep{Henao2011}. While this strategy mitigates the error resulting from the high dimensionality of a large system, it also introduces discrepancies among these units. Although the impact of violating constraints might be negligible in an isolated unit, the errors of intermediate variables can propagate throughout the entire system and be magnified over time \citep{Ma2022a}. Therefore, the violation of mass balance becomes unacceptable when multiple surrogate models are interconnected. 

This calls for the development of PINN architectures that embed hard constraints (hPINN) rather than soft constraints. A few studies claiming to enforce hard constraints fail to provide rigorous mathematical guarantees \citep{Dener2020, Lu2021}. 
The realization of ``hard constraints'' in these architectures heavily relies on trial-and-error tuning of hyperparameters. Hyperparameters that are set without adequate experience can easily lead to the failure of hard constraints and even poor predictions. In this work, we develop a novel PINN architecture with two non-trainable layers. The two layers equivalently represent an orthogonal projection of model predictions onto a feasible region of predefined linear equality constraints. This projection can be formulated as a quadratic program (QP) and analytically solved by the KKT conditions within the architecture. We hence refer to it as KKT-hPINN, as it is grounded in an analytical solution that always satisfies hard linear equality constraints in both training and testing processes. It does not require additional hyperparameters and does not increase computational cost. Results from three case studies show that KKT-hPINN can outperform neural networks and PINN when serving as unit-level, subsystem-level, and plant-level surrogate models in process system engineering. 

The remainder of the paper is organized as follows. In Section \ref{sec:Literature review}, we begin with a literature review of data-driven methods used for surrogate models, as well as their pros and cons. Section \ref{sec:Methodology} describes the proposed neural network architecture, KKT-hPINN, that strictly imposes hard linear equality constraints. We present the case studies in Section \ref{sec:Case Studies} and demonstrate its superiority over unconstrained neural networks and soft-constrained PINN. It is observed that KKT-hPINN not only imposes hard linear equality constraints in practice but also improves the predictive capability. In the end, Section \ref{sec:Conclusions} concludes the essential characteristics of the KKT-hPINN. 

\section{Literature review}
\label{sec:Literature review}
Numerous studies have been conducted to formulate data-driven approaches for modeling input-output relationships, with some being applied in chemical engineering. This section seeks to summarize the surrogate modeling techniques that have been extensively explored in this field, as well as some related work about PINN. 

Automated Learning of Algebraic Models (ALAMO) is a computational methodology to learn algebraic models by selecting the best subset of regressors \citep{Wilson2017, Cozad2014}. By minimizing metrics that are common in information theory and statistical theory, ALAMO learns algebraic functions that take into account both subset fitness and complexity. The outcome is in an enclosed functional form, which is expressed as a linear combination of basis functions. This approach yields a more tractable model and requires less data. ALAMO is distinguished as one of the surrogate modeling techniques developed by chemical engineers, featuring its success in reaction kinetics parameter estimation \citep{Wilson2019, Na2021} and process optimization \citep{Ma2022a, Ma2022b}. 

Gaussian process regression (GP) is a non-parametric and probabilistic approach that does not estimate parameters for a particular function. Instead, GP directly characterizes the distribution of the mapping function itself, providing the mean and covariance of model predictions. The function is regarded as a Gaussian process, wherein any finite set of random variables conforms to a joint Gaussian distribution. GP capitalizes on its capacity to infer the conditional predictive distribution using simulation data, circumventing a training process \citep{Rasmussen2005}. Incorporation of GP into optimization has been widely explored \citep{Wiebe2022, Schweidtmann2021, Quirante2015}. It has shown promising results in a broad range of applications including model predictive control (MPC) \citep{Paulson2021, Bonzanini2021}, erosion prediction \citep{Dai2022}, experimental design \citep{Olofsson2018}, biotechnology \citep{Mehrian2018}, pharmaceutical manufacturing \citep{Boukouvala2013}, and feasibility analysis \citep{Boukouvala2012}. Introducing physics-based knowledge into GP has also been investigated and showed improvements in prediction performance, but this still depends on penalization and soft constraints \citep{Kim2023}. A recent study also proposed a linearly constrained GP such that any samples and prediction can automatically fulfill known linear operator constraints \citep{Jidling2017}. 

Neural network is another popular approach characterized by its high flexibility and adaptability. The structure and mechanisms of NN will be explained in detail in Section \ref{sec:Methodology}. Neural networks have exhibited remarkable competence in approximating deterministic relationships between inputs and outputs across various domains. Breakthrough developments have been found in computer vision \citep{Krizhevsky2012, Szegedy2014}, speech recognition \citep{Hinton2012}, natural language processing \citep{Sutskever2014}, as well as protein structure prediction \citep{Jumper2021}. These applications harness neural networks to discover unknown relationships from data. On the other hand, neural networks have also been leveraged to substitute known physical models. In process system engineering, some works have been dedicated to accurately representing chemical processes \citep{Misener2023, Ma2022b, Goldstein2022, Henao2011}. The growing interest in surrogate models of chemical processes can be attributed to recent studies into the mixed-integer programming (MIP) formulations for neural networks \citep{Fischetti2018}. This exploration paves the way to integrate process optimization with the expressive power of neural networks, thereby facilitating decision-making over complex systems \citep{Grimstad2019, Anderson2020, Tsay2021, Schweidtmann2019}. An open-source package, OMLT, designed to automate the transformation of trained neural networks has been made available \citep{Ceccon2022}.  


PINN not only learns surrogate models as a class of function approximators but also integrates the physical constraints \citep{Karniadakis2021}. Hence, PINN has found extensive attention in surrogate modeling for multiphysics problems in which data are usually expensive. Unlike classic numerical methods, PINN typically targets ill-posed problems, where data under certain conditions are not precisely known due to difficulties in measurements in the real world. To address uncertainties stemming from imperfect data and stochastic physical systems, PINN has demonstrated its success in inferring solutions for both forward and inverse problems that involve partial differential equations (PDEs) \citep{Raissi2019}, and uncovering hidden physics in partially known stochastic ordinary differential equations (SDEs) \citep{OLeary2022}. Domain-specific applications include colloidal self-assembly \citep{Nodozi2023}, fluid mechanics \citep{Raissi2020}, thermodynamics \citep{Cai2021} and Model Predictive Control (MPC) \citep{Zheng2023a, Zheng2023b, Alhajeri2022}. These works leverage automatic differentiation and embed differential equations into the loss of a neural network as soft constraints. 

In parallel, a few studies have also attempted to integrate hard constraints into PINN but lack a mathematical guarantee of strict satisfaction \citep{Dener2020, Lu2021}. They adapt the augmented Lagrangian method to the training process, namely aug-Lag hPINN. In this approach, the inner loop addresses unconstrained optimization problems with conventional optimizers, while the outer loop updates the Lagrangian multipliers and penalty factors. Inappropriate initial values may lead to ill-conditioned optimization or being stuck at poor local minima. Technically, aug-Lag hPINN requires sufficient training time until the loss of constraints violation falls below a specific tolerance in every inner loop. This is achieved by a predetermined number of sub-iterations, which results in increased computing time. This approach relies on the assumption that the violation decreases to nearly zero within the inner loops.
\begin{figure}
    \centering
    \includegraphics[width=\linewidth]{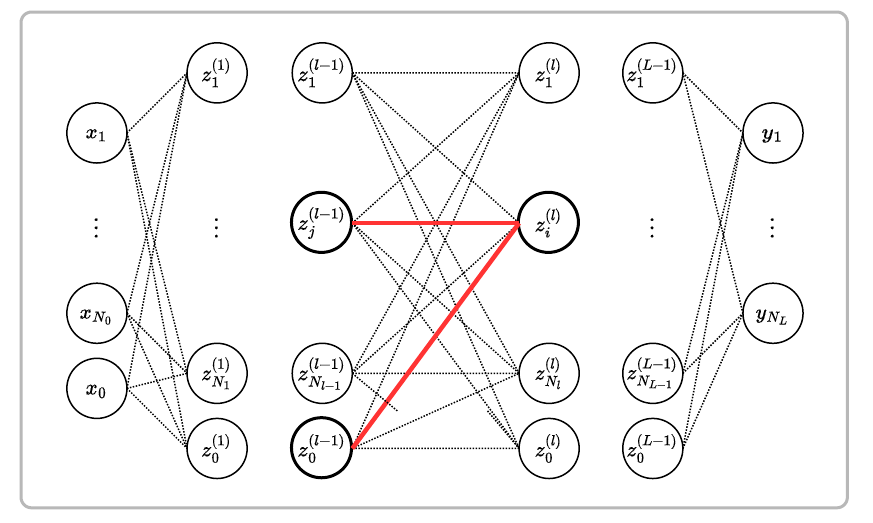}
    \caption{Fully connected feed-forward neural network architectures. The neuron $i$ at the $(l)$th layer is the linear combination of neurons at the $(l-1)$th layer followed by a nonlinear activation $\sigma$, which is
$z^{(l)}_i = \sigma(\sum_j w^{(l-1)}_{ij} z^{(l-1)}_j + b^{(l-1)}_{i0})$. The solid red lines represent $w_{ij}^{(l-1)}$ and $b_{i0}^{(l-1)}$ respectively.}
    \label{fig:NNarchitectures}
\end{figure}

\section{Methodology}
\label{sec:Methodology}
\subsection{Neural Network Structure}
A fully connected feed-forward neural network can be expressed as a multi-layer perceptron (MLP) model that maps input $\mathbf{x}$ to output $\mathbf{y}$ through a number of hidden layers as shown in Figure \ref{fig:NNarchitectures}. The forward equation in matrix form of a feed-forward neural network with $L+1$ layers is defined as
\begin{equation}
\mathbf{z}^{(l)} = \sigma(\mathbf{W}^{(l-1)} \mathbf{z}^{(l-1)} + \mathbf{b}^{(l-1)}) \quad \forall l = 0, ..., L
\end{equation}
where $\mathbf{z}^{(l)}$ denotes the output of the $l$-th layer, $\mathbf{z}^{(0)} = \mathbf{x} \in \mathcal{R}^{N_0}$ is the input, and $\mathbf{z}^{(L)} = \mathbf{\hat{y}} \in \mathcal{R}^{N_L}$ is the output. In every layer, the input neurons from the previous layer, $\mathbf{z}^{(l-1)}$, are weighted by $\mathbf{W}^{(l-1)}$ and biased by $\mathbf{b}^{(l-1)}$, followed by a nonlinear activation function $\sigma(\cdot)$. This results in the neurons, $\mathbf{z}^{(l)} \in \mathcal{R}^{N_l}$, on the $l$-th layer. 

The nature of training a neural network over $N$ training data points is to minimize a loss function over the weights and biases $ \{\mathbf{W}^l, \mathbf{b}^l\}_{l=0}^{L-1}$. The mean squared error (MSE) is one of the widely used loss functions for regression tasks. Denoting these trainable parameters as $\mathbf{\Theta} = \{\mathbf{W}^l, \mathbf{b}^l\}_{l=0}^{L-1}$, the solution to this minimization problem, $\theta^* = \operatorname{arg min}_{\theta \in \mathbf{\Theta}} J(\theta)$ where $J(\theta)$ is the loss function, can be iteratively approximated by a wide range of well-established algorithms. For first-order gradient-based optimizers, the learning direction is the gradient of the loss function with respect to parameters. The gradients can be accessed by chain rule and retrieved through automatic differentiation.

\subsection{Hard Linearly-Constrained Neural Network Architecture}
In certain applications, the relationships between some inputs $\mathbf{x}$ and outputs $\mathbf{y}$ may not be completely unknown. For instance, $
\mathbf{x}$ and $\mathbf{y}$ can contain some variables representing inflows and outflows that are subject to a mass or energy balance. The conservation between them is typically accessible using engineering principles and can be expressed as linear equality constraints, $\mathbf{Ax} + \mathbf{By} = \mathbf{b}$. In the present work, we propose a physics-informed neural network architecture that allows these linear equality constraints to always hold in both training and testing processes.

Before introducing the algebraic representation of the architecture, we first provide the geometric intuition. This specialized architecture can be viewed as a neural network connected to a corrector that adjusts the predictions deviating from linear equality constraints shown in Figure \ref{fig:projection}. For a given input $\hat{\mathbf{x}}$, the output from a neural network prediction $\hat{\mathbf{y}}$ may not satisfy the linear equality constraints. The corrector will output the orthogonal projection of $\hat{\mathbf{y}}$ onto the feasible region, denoted as $\tilde{\mathbf{y}}$.
\begin{figure}
    \centering
    \includegraphics[width=0.8\linewidth]{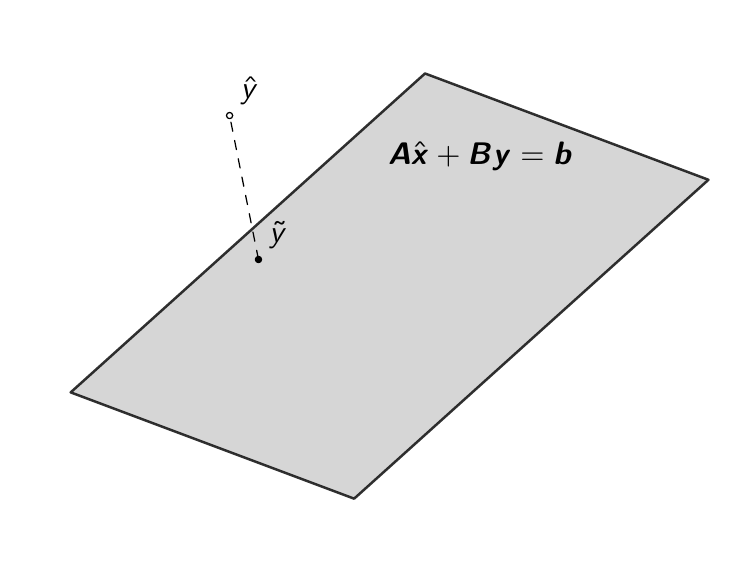}
    \caption{Illustration: for a given input $\mathbf{\hat{x}} \in \mathcal{R}^{N_0}$, neural network prediction $\mathbf{\hat{y}} \in \mathcal{R}^{N_L}$ is orthogonally projected to be $\mathbf{\tilde{y}} \in \mathcal{R}^{N_L}$ that satisfies a system of linear equality constraints $\mathbf{A} \hat{\mathbf{x}} + \mathbf{B y} = \mathbf{b}$. As an illustrative example, a hyperplane is used here to represent a single constraint where $\mathbf{A} \in \mathcal{R}^{1 \times N_0}, \mathbf{B} \in \mathcal{R}^{1 \times N_L}, \mathbf{b} \in \mathcal{R}^1$.}
    \label{fig:projection}
\end{figure}
\begin{figure}
    \centering
    \includegraphics[width=\linewidth]{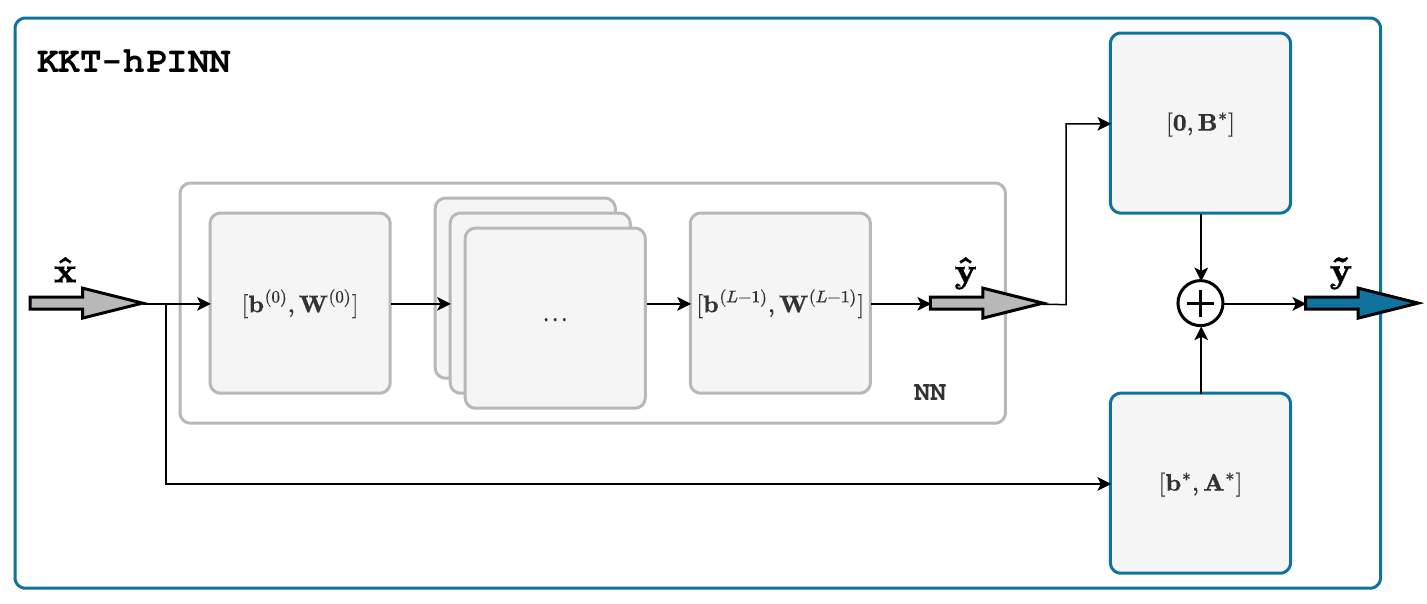}
    \caption{Grey block: illustration of NN architectures; Blue block: illustration of KKT-hPINN architectures consisting of trainable layers and two additional non-trainable projection layers (blue layers). The non-trainable parameters $\mathbf{A}^*$, $\mathbf{b}^*$ and $\mathbf{B}^*$ can be explicitly calculated from \eqref{KKTsolution}.}
    
    \label{fig:NNOPTblocks}
\end{figure}

Specifically, two separate, non-trainable, and fully connected layers are added to a neural network as the corrector, guiding its prediction to the closest feasible point. One layer is a linear transformation of specific input $\mathbf{\hat{x}}$ with fixed weights $\mathbf{A}^*$ and fixed bias $\mathbf{b}^*$, the other layer transforms the output $\mathbf{\hat{y}}$ predicted from the original model $\mathrm{NN} (\mathbf{\Theta}, \mathbf{\hat{x}})$ with a projection matrix $\mathbf{B}^*$. The definitions of $\mathbf{A}^*$, $\mathbf{b}^*$, and $\mathbf{B}^*$ will be specified in Theorem \ref{thm1}.  We name this architecture KKT-hPINN, as it is a hard linearly constrained physics-informed neural network derived from the KKT conditions. 
\begin{Theorem}
\label{thm1}
    Given any input $\mathbf{\hat{x}} \in \mathcal{R}^{N_0}$, a neural network model $\mathbf{\hat{y}} = \mathrm{NN} (\mathbf{\Theta}, \mathbf{\hat{x}}): \mathcal{R}^{N_0} \rightarrow \mathcal{R}^{N_L}$, and prior knowledge about $m$ equality constraints for the input $\mathbf{\hat{x}}$ and the ground truth $\mathbf{y}$: $\mathbf{A} \hat{\mathbf{x}} + \mathbf{By} = \mathbf{b}$ where $\mathbf{A} \in \mathcal{R}^{m \times N_0}, \mathbf{b} \in \mathcal{R}^m, \mathbf{B} \in \mathcal{R}^{m \times N_L}$ and $\mathrm{rank}(\mathbf{B})=m$, the orthogonal projection of $\mathbf{\hat{y}}$ onto a feasible region represented by a system of linear equality constraints, $\mathbf{\tilde{y}}$, has the following analytical solution:
\begin{equation}
\label{KKTsolution}
    \mathbf{\tilde{y}} = \mathbf{A^* \hat{x}} + \mathbf{B^* \hat{y}} + \mathbf{b^*},
\end{equation}
where
\begin{equation}
\begin{aligned}
    \mathbf{A^*} &= -\mathbf{B}^T(\mathbf{B}\mathbf{B}^T)^{-1}\mathbf{A} \nonumber\\
    \mathbf{B^*} &= \mathbf{I} - \mathbf{B}^T(\mathbf{B}\mathbf{B}^T)^{-1}\mathbf{B}\\
    \mathbf{b^*} &= \mathbf{B}^T(\mathbf{B}\mathbf{B}^T)^{-1}\mathbf{b}
\end{aligned}
\end{equation}
\end{Theorem}
\begin{proof}
    Correcting predictions that deviate from constraint satisfaction with minimum Euclidean distance can be conceptualized as an orthogonal projection onto the feasible region of physical constraints, as depicted in Figure \ref{fig:projection}. This can be formulated as the following quadratic programming problem after the forward calculation of $\mathbf{\hat{y}}$ in each iteration:
\begin{equation}
\label{OptLayerFormulation}
\begin{aligned}
    \mathbf{\tilde{y}} = \operatorname{argmin}_{\mathbf{y}} \quad & \frac{1}{2} \lVert \mathbf{y} - \mathbf{\hat{y}}\rVert^2 \quad \text{s.t.} \quad \mathbf{A \hat{x}} + \mathbf{By} = \mathbf{b}
\end{aligned}
\end{equation}
where $\mathbf{A}$, $\mathbf{B}$ and $\mathbf{b}$ are constant matrices and vector involved in certain prior knowledge. That is, instead of feeding the prediction $\mathbf{\hat{y}}$ directly into the loss function for backpropagation, we first project it onto a feasible region where the linear equality is strictly satisfied.

Because of the convexity of problem \ref{OptLayerFormulation}, the optimal primal and dual solutions $\mathbf{\tilde{y}}$ and $\boldsymbol{\lambda}^*$ can be found by solving the KKT conditions below.
\begin{equation*}
    \begin{bmatrix}
        \mathbf{I} & \mathbf{B}^T \\
        \mathbf{B} & \mathbf{0}
    \end{bmatrix} 
    \begin{bmatrix}
        \mathbf{\tilde{y}} \\
        \boldsymbol{\lambda}^*
    \end{bmatrix}
    =
    \begin{bmatrix}
        \mathbf{\hat{y}} \\
        \mathbf{b} - \mathbf{A \hat{x}}
    \end{bmatrix}
\end{equation*}
In the neural network setting, the optimal solution $\mathbf{\tilde{y}}$ in \eqref{OptLayerFormulation} can be expressed as the sum of outputs from two additional projection layers with fixed parameters $\mathbf{A^*}$, $\mathbf{B^*}$ and $\mathbf{b^*}$. 
\end{proof}

This architecture design presents a distinctive methodology with key features outlined in the following remarks:
\begin{remark}[Applicability to other architectures]
   The proposed projection layer can be applied to any existing neural network architectures, such as convolutional neural networks and recurrent neural networks, instead of being limited to MLP. 
\end{remark}

\begin{remark}[Compatibility to PINN]
\label{rmk:compatibility}
    The KKT-hPINN can be compatible with the PINNs, i.e., enforcing hard linear equality constraints and soft nonlinear constraints simultaneously. This compatibility arises because the hard constraints are introduced using a different architecture, while the soft constraints are embedded into the loss function.
\end{remark}

\begin{remark}[Difference with post projection]
\label{rmk:postprojection}
    It is worth noting that embedding the projection layers within the architecture should be distinguished from trivially applying orthogonal projection to the predictions obtained from a neural network in the testing process. The latter strategy, namely post-projection during testing, does not take into account the ground truths of the output in the training data when the performing projection step and thus may compromise the accuracy of the predictions. In contrast, the presence of projection layers not only imposes hard constraints but also changes how the model learns in every iteration of the training process. A quantitative explanation is that the projection layers will change the gradient descent direction, leading to
    different parameters for KKT-hPINN compared to neural networks without the projection layers.
\end{remark}

\begin{remark}[Loss functions]
    The KKT-hPINN architecture changes the loss function so that the learning direction, the gradient of the loss function, is also changed. When the MSE is selected as the criterion of the loss function, redefine $\mathbf{z}^{(l)}:= [1, \mathbf{z}^{(l)}] \in \mathcal{R}^{N_l + 1}$ and concatenate bias $\mathbf{b}^{(l)}$ with weight $\mathbf{W}^{(l)}$ at each layer as $\boldsymbol{\theta}^{(l)} = [\mathbf{b}^{(l)}, \mathbf{W}^{(l)}]$ for notation simplicity, the optimization problems solved in the training of the NN, PINN, and KKT-hPINN can be expressed as
    \begin{align}
& \text{NN:} \quad \operatorname{min}_{\mathbf{\Theta}} \quad
    \frac{1}{2N} \sum_{i=1}^{N} \lVert \mathrm{NN} (\mathbf{\Theta}, \mathbf{x}^i) - \mathbf{y}^i\rVert^2 \\
& \text{PINN:} \quad \operatorname{min}_{\mathbf{\Theta}} \quad
    \frac{1}{2N} \sum_{i=1}^{N} (\lVert \mathrm{NN} (\mathbf{\Theta}, \mathbf{x}^i) - \mathbf{y}^i\rVert^2  + \lambda \lVert \mathbf{A} \mathbf{x}^i + \mathbf{B} \mathrm{NN} (\mathbf{\Theta}, \mathbf{x}^i) - \mathbf{b} \rVert^2) 
    \label{eq:PINN loss fn}\\
&  \text{KKT-hPINN:} \quad \operatorname{min}_{\mathbf{\Theta}} \quad
    \frac{1}{2N} \sum_{i=1}^{N} \lVert \mathbf{A^*}\mathbf{x}^i + \mathbf{B^*} \mathrm{NN}(\mathbf{\Theta}, \mathbf{x}^i) + \mathbf{b^*} - \mathbf{y}^i\rVert^2
\end{align}
 KKT-hPINN enforces the model to update the parameter in a way such that inherent linear patterns present within the simulation data are always satisfied. It is anticipated that the convergence can be improved by this characteristic of KKT-hPINN.
\end{remark}

\section{Case Studies}
\label{sec:Case Studies}
 This section covers case studies of a CSTR unit, a DME-DEE chemical plant, and an extractive distillation subsystem. The performance of the neural network (NN), the physics-informed neural network (PINN), the neural network with post projection in Remark \ref{rmk:postprojection} (NNPost), and the proposed KKT-hPINN are examined and compared. 

 \begin{table}[ht]
    \centering
    \begin{tabular}{lccc}
    \toprule
          Case&CSTR&  Plant&  Distillation\\
          \midrule
          Hidden layers&2&  2&  2\\
          \midrule
          Neurons/layer&12&  32&  32\\
          \midrule
          Learning rate&$10^{-4}$ &  $10^{-4}$ &  $10^{-4}$ \\
          \midrule
          Batch size&16&  16&  16\\
          \midrule
          Samples&$\sim 1500$ &$\sim 1200$  &$\sim 5000$  \\
          \bottomrule
    \end{tabular}
    \caption{Neural network architectures and hyperparameter settings}
    \label{table:hyperparameters}
\end{table}

\paragraph{Experiment setup}
The datasets are generated from simulators constructed in Aspen Plus V.11. The Aspen simulator for the CSTR unit is customized with a built-in model, while those for the chemical plant and the extractive distillation subsystem can be found in the supplementary files of \cite{Ma2022b}. The preliminary data obtained from Aspen may not be valid due to unconvergence or numerical issues during the simulations. To ensure validity, all the data marked with severe errors are eliminated. After that, the data are filtered if they violate a certain tolerance for specific linear equality constraints in the case studies. The tolerances for the data of the CSTR unit, the DME-DEE chemical plant, and the extractive distillation unit are set to be $10^{-8}$, $10^{-6}$, and $10^{-8}$ respectively. 

The data from the feasible simulations is stored and scaled using maximum absolute scaling, i.e., dividing the data by the maximum absolute value of each variable in the dataset. The surrogate models are built using the PyTorch package in Python \citep{Paszke2019}. The neural network architectures used in the case studies, along with the hyperparameter settings, are documented in Table \ref{table:hyperparameters}. NN, PINN, NNPost, and KKT-hPINN adopt the same settings in each case study. 

The models are trained and evaluated with respect to Root Mean Squared Error (RMSE). During training, the parameters are updated in every epoch to minimize the RMSE value between the current prediction and ground truth. A validation dataset is used to evaluate the extent of overfitting in models during each epoch. This assesses the RMSE value without influencing parameter updates. By default, $20\%$ samples in Table \ref{table:hyperparameters} are used for validation. A well-trained model is expected to maintain a negligible gap between training and validation RMSE values, with both values being consistently low. 

The well-trained model is then applied to an unseen test dataset, where the RMSE values can vary inherently in different runs, because training a neural network cannot guarantee global optimality. To ensure the robustness of the results, this process is repeated 10 times. Besides, generating high-fidelity data can be challenging and time-consuming in some fields. To investigate the effect of the size of the training dataset, models are trained using a different number of training samples and compared in terms of their improvements relative to a tranditional neural network. 

\begin{figure}
    \centering
    \includegraphics[width=\linewidth]{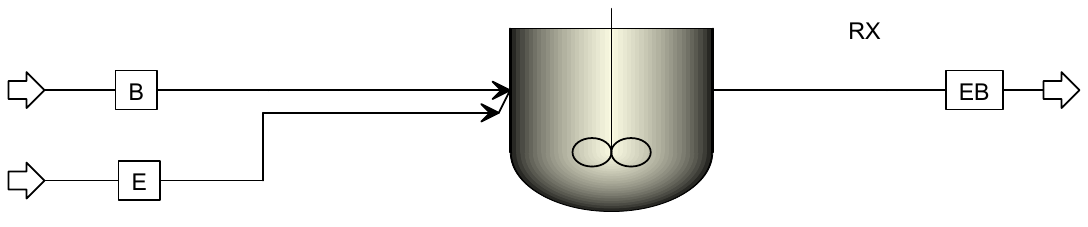}
    \caption{Process flowsheet of the CSTR unit.}
    \label{fig:cstrflowsheet}
\end{figure}

\subsection{CSTR Unit}\footnote{The proposed model and case studies are made available on our GitHub repository at: \url{https://github.com/li-group/KKThPINN.git}}
To illustrate the construction of a KKT-hPINN, a simple simulator of CSTR as shown in Figure \ref{fig:cstrflowsheet} is considered here. Suppose a CSTR converts reactants, Benzene (B) and Ethylene (E), into the product, Ethylbenzene (EB), and the process follows a first-order chemical reaction: 
\begin{equation}
\label{reactioneq}
    B + E \rightarrow EB
\end{equation}

\begin{table}
  \centering
  \begin{tabular}{lc}
    \toprule
    Variables & Description \\
    \midrule
    $x_1$ & Temperature of the RX unit \\
    \midrule
    $x_2$ & Molar flow rate of Benzene in the B stream \\
    \midrule
    $x_3$ & Molar flow rate of Ethylene in the E stream \\
    \midrule
    $y_1$ & Molar flow rate of Ethylbenzene in the EB stream \\
    \midrule
    $y_2$ & Molar flow rate of Benzene in the EB stream \\
    \midrule
    $y_3$ & Molar flow rate of Ethylene in the EB stream  \\
    \midrule
    $\mathbf{Ax} + \mathbf{By} = \mathbf{b}$ & 
    $\begin{cases}
        x_2 - x_3 - y_2 + y_3 = 0,  & \text{Reactants consumption}\\
        x_2 - y_1 - y_2 = 0,  & \text{EB production}
    \end{cases}$ \\ 
    \bottomrule
  \end{tabular}
  \caption{Variables and constraint of the CSTR unit.}
  \label{cstrvariables}
\end{table}
Pure reactants are fed at a steady state. The CSTR has fixed volume and working pressure. The variables left for design are the molar flow rate of benzene and ethylene, and the working temperature of CSTR. The goal is to leverage the neural network as a universal approximator to train a surrogate model that predicts the molar flow of ethylene, benzene, and ethylbenzene in the output stream. In practice, underlying equations of a physical model are usually much more complicated than the equation of reaction kinetics for this CSTR unit, which necessitates the use of a surrogate model.  

\begin{figure}
    \centering
    \includegraphics[scale=0.8]{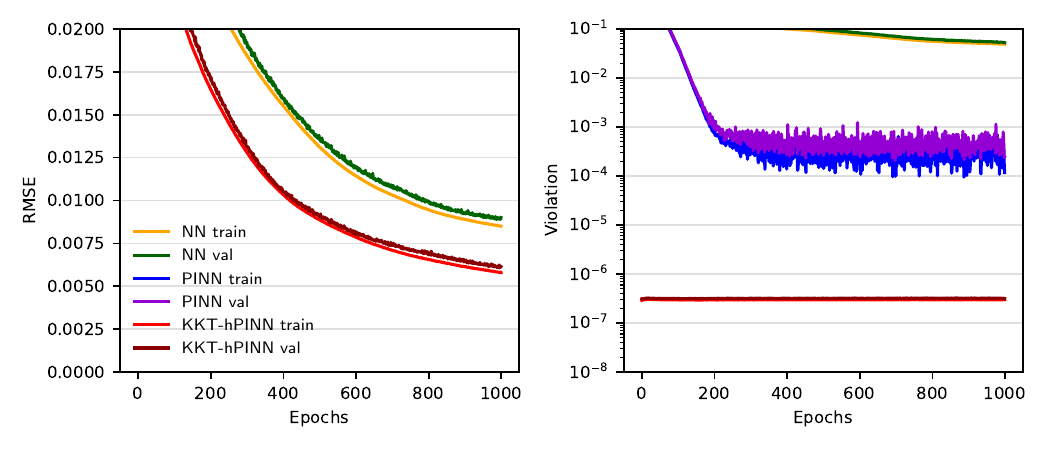}
    \caption{Learning curve of the CSTR surrogate models. Left: training RMSE and validation RMSE (RMSE values for PINN are outside the limits); Right: the magnitude of violation of the linear equality constraints.}
    \label{fig:LearningCSTR}
\end{figure}

Two linear equations related to conservation properties are applicable to the input and the output streams: From the stoichiometry in Equation \eqref{reactioneq}, it can be readily deduced that the consumption of benzene is equal to the consumption of ethylene and the production of ethylbenzene, which is expressed in the Table \ref{cstrvariables}. It should noted that linear equality constraints embedded here should be independent according to Theorem \ref{thm1}. Therefore, dependent constraints, such as the relationship between the consumption of ethylene and the production of ethylbenzene, are not included. With the prior knowledge of $\mathbf{A}$, $\mathbf{B}$ and $\mathbf{b}$ in Table \ref{cstrvariables}, the KKT-hPINN is able to find the projection matrices and keep projecting the predictions onto feasible regions.

\paragraph{Comparison of training and validation loss} The trend of RMSE over epochs in Figure \ref{fig:LearningCSTR} illustrates the generalizability of NN, PINN and KKT-hPINN models to the validation dataset. The RMSE value of the KKT-hPINN becomes lower than that of the NN within just 1000 epochs. For the PINN, the penalty term in the loss function in the \eqref{eq:PINN loss fn} is excluded from RMSE calculation for a fair comparison. The presence of soft constraints dramatically hinders convergence of the PINN, leading to RMSE values that are notably outside the limits depicted in this figure.

\paragraph{Comparison of the violation of constraints} By substituting the predictions into the linear equality constraints, Figure \ref{fig:LearningCSTR} also indicates the extent to which they violate the conservation laws that are established as constraints. It is found that the stoichiometry does not hold in the NN model, with the residual of constraints being around $10^{-1}$. Given the fact that the data has been scaled, $10^{-1}$ is a significant violation. Even though much higher RMSE values are found in the PINN, it does result in a lower violation fluctuating around $10^{-4}$ because of the presence of penalty term in the loss function. On the contrary, the violation in the KKT-hPINN is always controlled at $10^{-7}$, which adequately meets practical needs. It should be noted that PyTorch tensors operate as float32 by default, which could potentially result in a lower numerical precision. This accounts for the negligible decrease in the violation value from the tolerance of $10^{-8}$ to $10^{-7}$. 

\begin{figure}
    \centering
    \includegraphics[width=\linewidth]{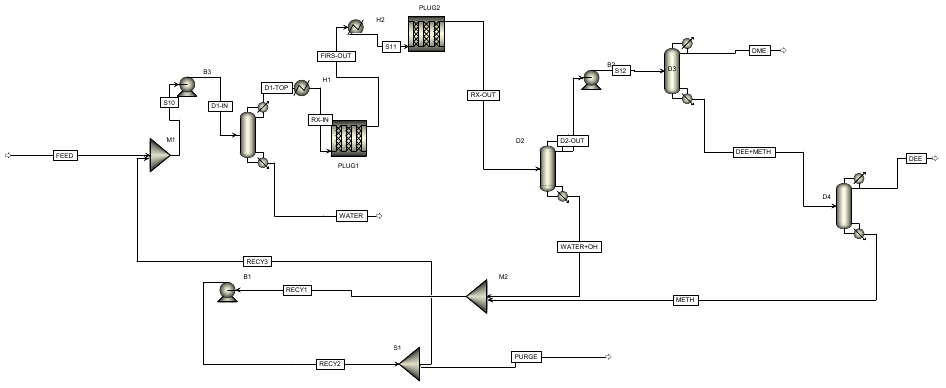}
    \caption{Process flowsheet of the DME-DEE chemical plant.}
    \label{fig:plantflowsheet}
\end{figure}

\subsection{DME-DEE Chemical Plant}
This case study focuses on a chemical plant where methanol, ethanol, and water are used as feed materials to manufacture DME and DEE as shown in Figure \ref{fig:plantflowsheet}. The feed is pre-processed to remove waste water and then transferred into reactor units. Following this, the products and reactants are separated and refined, with the heavier components recycled back into the preprocessing units. For illustrative simplicity, we assume that all the units have been designed. In other words, the entire system is fixed and
we aim to build a surrogate model that predicts the outflow given the inflow and recycling, treating the system as a black box. Nevertheless, it is straightforward to conclude that the flow-in is equal to the flow-out by the mass balance in Table \ref{plantvariables}, and this serves as the single hard constraint in this model. 

\begin{table}[ht]
  \centering
  \scalebox{0.90}{
  \begin{tabular}{lc}
    \toprule
    Variables & Description \\
    \midrule
    $x_1$ & Mass flow rate of methanol in the FEED stream \\
    \midrule
    $x_2$ & Mass flow rate of ethanol in the FEED stream \\
    \midrule
    $x_3$ & Mass flow rate of water in the FEED stream \\
    \midrule
    $x_4$ & Total mass flow rate of the PURGE stream \\
    \midrule
    $y_1$ & Total mass flow rate of the DME stream \\
    \midrule
    $y_2$ & Mass flow rate of DME in the DME stream \\
    \midrule
    $y_3$ & Total mass flow rate of the DEE stream \\
    \midrule
    $y_4$ & Mass flow rate of DEE in the DEE stream \\
    \midrule
    $y_5$ & Total mass flow rate of the WATER stream \\
    \midrule
    $\mathbf{Ax} + \mathbf{By} = \mathbf{b}$ & 
    $\begin{cases}
        x_1 + x_2 + x_3 - x_4 - y_1 - y_3 - y_5 = 0,  & \text{Mass balance}\\
    \end{cases}$ \\ 
    \bottomrule
  \end{tabular}}
  \caption{Variables and constraint of the DME-DEE chemical plant.}
  \label{plantvariables}
\end{table}

\paragraph{Comparison of training and validation loss} Figure \ref{fig:LearningPlant} shows that all the models can be generalized to the validation dataset. The deterministic relationship in this dataset is easier to represent, as all models yield very low RMSE values. In spite of this, KKT-hPINN is still trained to a lower loss in the same number of epochs. 

\begin{figure}
    \centering
    \includegraphics[scale=0.8]{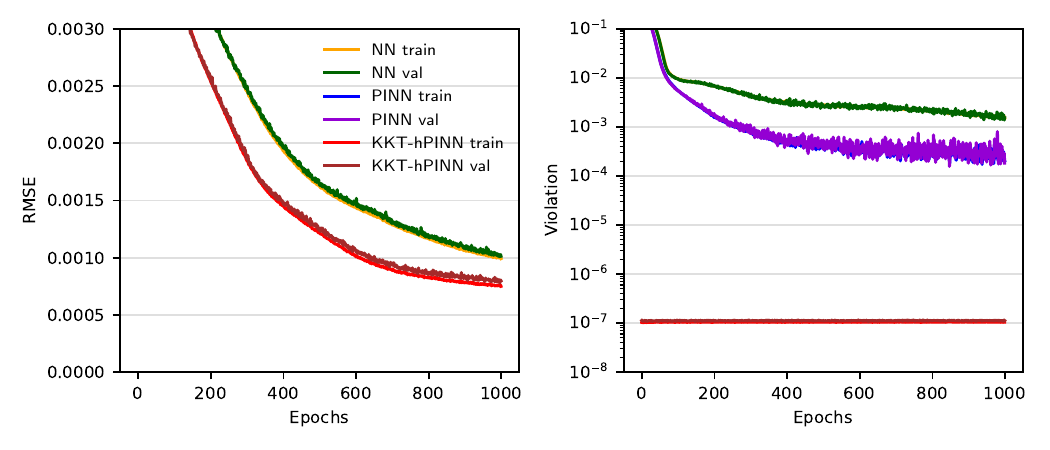}
    \caption{Learning curve of the DME-DEE chemical plant surrogate models. Left: training RMSE and validation RMSE (RMSE values for PINN are outside the limits); Right: the magnitude of violation of the linear equality constraints.}
    \label{fig:LearningPlant}
\end{figure}

\begin{table}
  \caption{Average RMSE $(\times 10^{-3})$ over 10 runs in DME-DEE Chemical Plant test dataset in 1000 epochs. Overall: RMSE for all variables; 
  Constrained: RMSE for the variables involved in the linear equality constraints; Unconstrained: RMSE for the variables not involved in the linear equality constraints.}
  \centering
  \begin{tabular}{llll}
  \toprule
  \cmidrule(r){1-2}
    Model       & Overall               & Constrained           & Unconstrained      \\
    \midrule
    NN   
    & $1.039 \pm 0.262$ & $0.986 \pm 0.257$ & $1.090 \pm 0.254$\\
    NNPost          
    & $1.015 \pm 0.249$ & $0.964 \pm 0.251$ & $1.086 \pm 0.252$\\
    PINN 
    & $4.203 \pm 0.333$ & $4.187 \pm 0.271$ & $4.188 \pm 0.620$\\
    KKT-hPINN 
    & $0.767 \pm 0.082$ & $0.713 \pm 0.073$ & $0.842 \pm 0.091$\\
    \bottomrule
  \end{tabular}
  \label{TablePlant}
\end{table}

\paragraph{Comparison of the violation of constraints} The magnitude of violation of mass balance is strictly maintained around $10^{-7}$ in the KKT-hPINN model. On the other hand, the predictions of NN and PINN cause higher violations, both around $10^{-3}$. In this case, PINN does not exhibit a significant improvement over NN. This implies that determining the penalty hyperparameter for soft constraints heavily depends on trial and error. Inappropriate hyperparameters can result in slow convergence or significant violations. 

\begin{figure}
    \centering
    \includegraphics[scale=0.8]{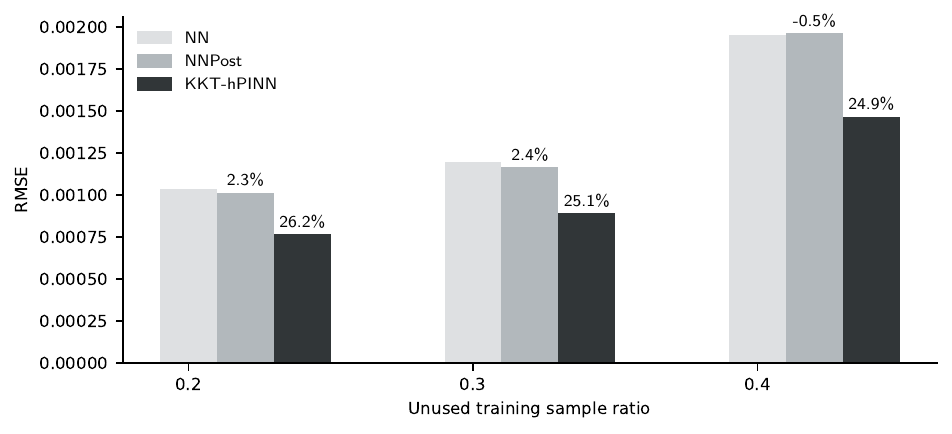}
    \caption{Bar: Average RMSE of KKT-hPINN and NN post-project over 10 runs in DME-DEE Chemical Plant test dataset. 20\%, 30\% and 40\% samples in the dataset are not used in the training process. Percentage: improvement of KKT-hPINN and NNPost relative to NN}
    \label{fig:BarPlant}
\end{figure}

\begin{figure}
    \centering
    \includegraphics[width=\linewidth]{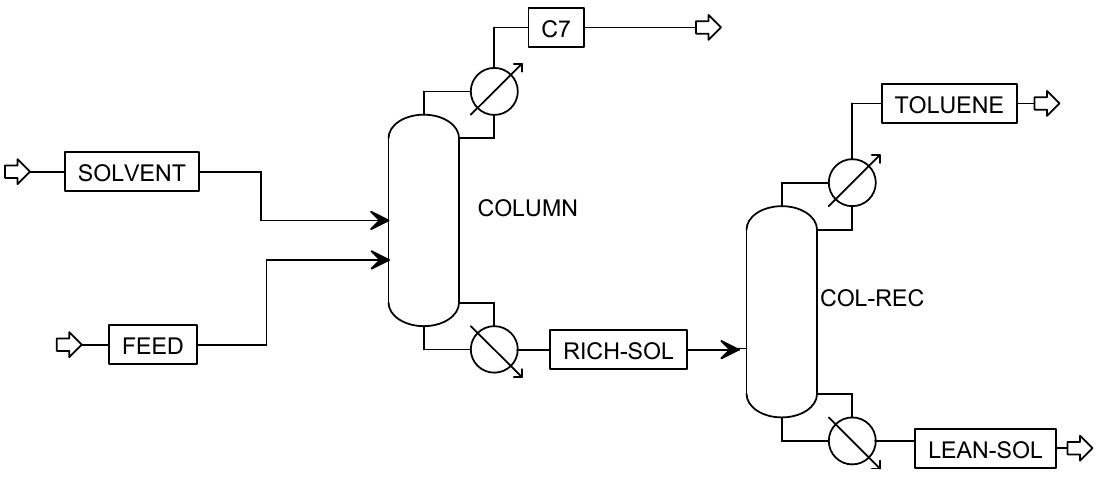}
    \caption{Process flowsheet of the extractive distillation subsystem.}
    \label{fig:distillationflowsheet}
\end{figure}

\paragraph{Test performance} In Table \ref{TablePlant}, it is observed that KKT-hPINN remains the superior model among the others with the lowest test RMSE values. An interesting observation is that KKT-hPINN not only improves the prediction of outputs that are subject to the constraints but also lowers the loss of outputs that are not involved in the constraints. This suggests that the addition of two projection layers provides the optimizer with a more physically feasible learning direction, resulting in the parameters that cause smaller errors. 

Experiments with different training dataset sizes are conducted 10 times for each. The result in Figure \ref{fig:BarPlant} shows that, regardless of the number of training samples, the KKT-hPINN model always decreases the overall test loss by around $25\%$ relative to the NN. On the other hand, 
the post-projection cannot bring any significant improvement and sometimes may even cause decline. This aligns with Remark \ref{rmk:postprojection}, highlighting the sharp distinction between post-projection after the training process and continuous projection during the training process. In addition, the PINN greatly deteriorates the model prediction so that it is not compared in this figure.  

\subsection{Extractive Distillation Subsystem}
\begin{table}[ht]
  \centering
  \scalebox{0.90}{
  \begin{tabular}{lc}
    \toprule
    Variables & Description \\
    \midrule
    $x_1$ & Molar flow rate of phenol in the SOLVENT stream \\
    \midrule
    $x_2$ & Reflux ratio of COLUMN column \\
    \midrule
    $x_3$ & Distillate rate of COLUMN column \\
    \midrule
    $x_4$ & Reflux ratio of COL-REC column \\
    \midrule
    $x_5$ & Distillate rate of COL-REC column \\
    \midrule
    $y_1$ & Molar flow rate of n-heptane in the C7 stream \\
    \midrule
    $y_2$ & Molar flow rate of toluene in the TOLUENE stream \\
    \midrule
    $y_3$ & Condenser heat duty of COLUMN column \\
    \midrule
    $y_4$ & Reboiler heat duty of COLUMN column \\
    \midrule
    $y_5$ & Condenser heat duty of COL-REC column \\
    \midrule
    $y_6$ & Reboiler heat duty of COL-REC column \\
    \midrule
    $y_7$ & Molar flow rate of toluene in the C7 stream \\
    \midrule
    $y_8$ & Molar flow rate of phenol in the C7 stream \\
    \midrule
    $y_9$ & Molar flow rate of n-heptane in the TOLUENE stream \\
    \midrule
    $y_{10}$ & Molar flow rate of phenol in the TOLUENE stream \\
    \midrule
    $\mathbf{Ax} + \mathbf{By} = \mathbf{b}$ & 
    $\begin{cases}
        x_3 - y_1 - y_7 - y_8 = 0,  & \text{C7 fractions}\\
        x_5 - y_2 - y_9 - y_{10} = 0,  & \text{TOLUENE fractions}
    \end{cases}$ \\ 
    \bottomrule
  \end{tabular}}
  \caption{Variables and constraint of the extractive distillation subsystem.}
  \label{distillationvariables}
\end{table}

Here, a case study of an extractive distillation subsystem as shown in Figure \ref{fig:distillationflowsheet} is presented, in which the design includes operating specifications for two distillation columns and the flow rate of the solvent. The Feed stream is flowing at a constant rate and consists of a 50/50 mixture of n-heptane and toluene. Phenol, a solvent, is used to aid in the separation of the azeotropic mixture of n-heptane and toluene. The first distillation column is primarily responsible for the separation of n-heptane, which is collected at the top. Following that, the second distillation column primarily separates toluene, while the phenol solvent is recovered at the bottom. The internal design of distillation columns, including the number of stages and the type of reboiler and condenser, has already been established. However, the optimal operating specifications have not been determined. In this case, the distillate rate and reflux ratio are chosen to specify the distillation columns. The goal is to train a surrogate model that predicts heat duties and flow rates of each component in the distillate streams. As no chemical reaction occurs in the distillation process, the sum of the molar flow rate of n-heptane, toluene and phenol is always equal to the distillate rate as shown in Table \ref{distillationvariables}.

\paragraph{Comparison of training and validation loss} 
All three models reach the plateau in about 1000 epochs, and KKT-hPINN achieves a lower RMSE value. As observed in Figure \ref{fig:LearningDistillation}, the RMSE values of these models exhibit considerable differences. The hard constraints still decrease the losses of KKT-hPINN, improving its ability to generalize effectively to noise-free simulation data. 

\begin{figure}
    \centering
    \includegraphics[scale=0.8]{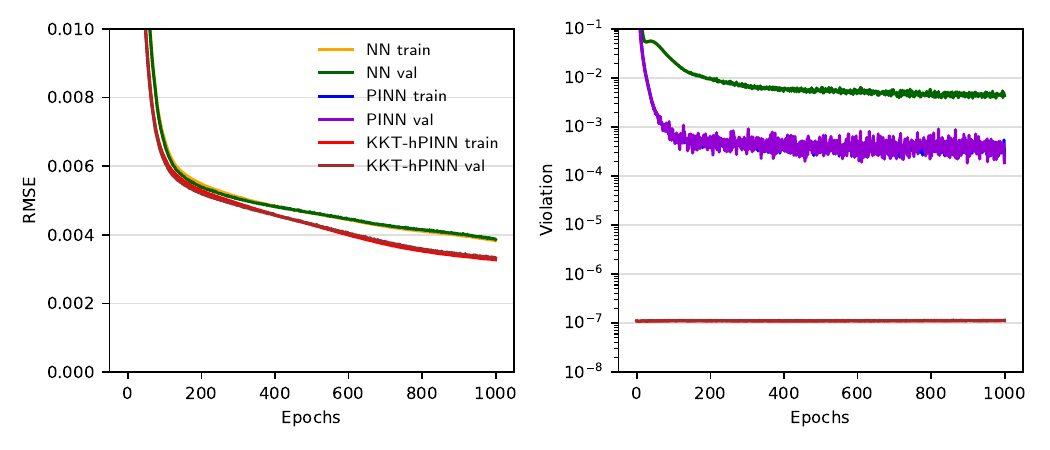}
    \caption{Learning curve of the extractive distillation surrogate models. Left: training RMSE and validation RMSE (RMSE values for PINN are outside the limits); Right: the magnitude of violation of the linear equality constraints.}
    \label{fig:LearningDistillation}
\end{figure}

\paragraph{Comparison of the violation of constraints} The result in Figure \ref{fig:LearningDistillation} shows that the KKT-hPINN still adheres strictly to the hard constraints, whereas violations of PINN fluctuate near $10^{-4}$. Even though the penalty hyperparameter chosen in this run can reduce the violation values to $10^{-4}$ in a small number of epochs, they reach a plateau for the remaining epochs. 

\begin{table}
  \caption{Average RMSE $(\times 10^{-3})$ over 10 runs in extractive distillation test dataset in 1000 epochs. Overall: RMSE for all variables; 
  Constrained: RMSE for the variables involved in the linear equality constraints; Unconstrained: RMSE for the variables not involved in the linear equality constraints.}
  \centering
  \begin{tabular}{llll}
  \toprule
  \cmidrule(r){1-2}
    Model       & Overall               & Constrained           & Unconstrained      \\
    \midrule
    NN   
    & $4.288 \pm 0.894$ & $5.070 \pm 1.209$ & $2.676 \pm 0.215$\\
    NNPost          
    & $4.283 \pm 0.894$ & $5.064 \pm 1.209$ & $2.676 \pm 0.215$\\
    PINN 
    & $15.549 \pm 2.039$ & $18.050 \pm 2.389$ & $10.623 \pm 2.171$\\
    KKT-hPINN 
    & $3.504 \pm 0.202$ & $4.090 \pm 0.239$ & $2.359 \pm 0.208$\\
    \bottomrule
  \end{tabular}
  \label{TableDistillation}
\end{table}

\begin{figure}
    \centering
    \includegraphics[scale=0.8] {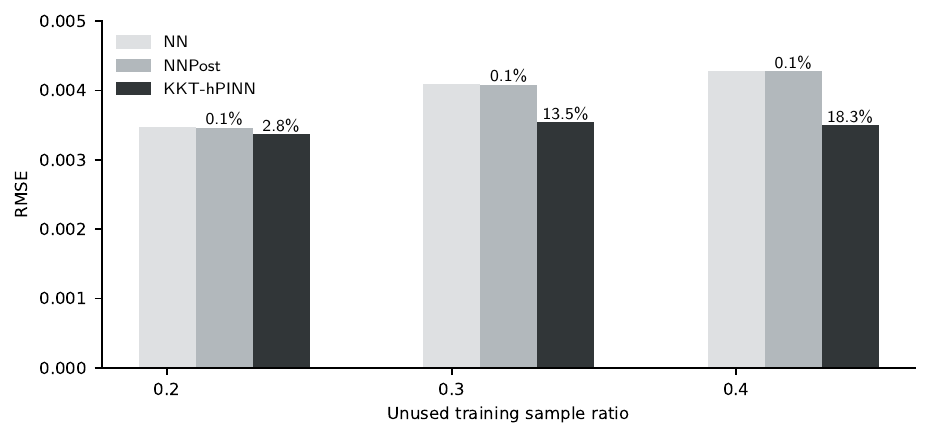}
    \caption{Bar: Average RMSE over 10 runs in extractive distillation test dataset. 20\%, 30\% and 40\% samples in the dataset are not used in the training process. Percentage: improvement of KKT-hPINN and NNPost relative to NN}
    \label{fig:BarDistillation}
\end{figure}

\paragraph{Test performance} The Table \ref{TableDistillation} also indicates that PINN and KKT-hPINN have the highest and the lowest RMSE values respectively. This suggests that the choice of an appropriate penalty value for PINN is critical, while KKT-hPINN guarantees hard constraints using an analytical solution. 

It is interesting that the RMSE values of all three models are almost at the same level when only $20\%$ of the data is not used for training in Figure \ref{fig:BarDistillation}. This may be explained by the fact that the dataset in this particular example contains a larger number of samples that are relatively sufficient for NN to represent the subsystem and become as competent as KKT-hPINN. However, when data are in shortage and hard to represent, projecting unconstrained predictions onto feasible regions where the data are known to reside becomes beneficial again. 

\section{Conclusions}
\label{sec:Conclusions}
Process optimization necessitates a more advanced surrogate model that can strictly enforce hard constraints and exhibit superior accuracy. A physics-informed neural network with hard linear equality constraints, KKT-hPINN, is proposed in this work. By introducing projection layers that embody an analytical solution derived from KKT conditions, hard linear equality constraints are rigorously guaranteed. They guide the parameters to update towards a more physically consistent and more accurate solution. Three case studies have been presented to illustrate its utilization as unit-level, subsystem-level, and plant-level surrogate models in process systems. In all cases, it is observed that the loss decreases more rapidly in KKT-hPINN, leading to higher accuracy compared to other neural networks. More importantly, all numerical results align with the mathematical guarantee of KKT-hPINN and distinguish it from the neural network and PINN. Overall, KKT-hPINN is a straightforward approach that is suitable for high-fidelity surrogate modeling and maximizes the use of evident conservation constraints. KKT-hPINN preserves the end-to-end benefits of neural network and potentially mitigates the error propagation caused by soft constraints, which renders it a prospective technique in the realm of process optimization.



\bibliographystyle{elsarticle-harv} 
\bibliography{main}

\begin{thebibliography}{52}
\expandafter\ifx\csname natexlab\endcsname\relax\def\natexlab#1{#1}\fi
\providecommand{\url}[1]{\texttt{#1}}
\providecommand{\href}[2]{#2}
\providecommand{\path}[1]{#1}
\providecommand{\DOIprefix}{doi:}
\providecommand{\ArXivprefix}{arXiv:}
\providecommand{\URLprefix}{URL: }
\providecommand{\Pubmedprefix}{pmid:}
\providecommand{\doi}[1]{\href{http://dx.doi.org/#1}{\path{#1}}}
\providecommand{\Pubmed}[1]{\href{pmid:#1}{\path{#1}}}
\providecommand{\bibinfo}[2]{#2}
\ifx\xfnm\relax \def\xfnm[#1]{\unskip,\space#1}\fi
\bibitem[{Alhajeri et~al.(2022)Alhajeri, Abdullah, Wu and
  Christofides}]{Alhajeri2022}
\bibinfo{author}{Alhajeri, M.S.}, \bibinfo{author}{Abdullah, F.},
  \bibinfo{author}{Wu, Z.}, \bibinfo{author}{Christofides, P.D.},
  \bibinfo{year}{2022}.
\newblock \bibinfo{title}{Physics-informed machine learning modeling for
  predictive control using noisy data}.
\newblock \bibinfo{journal}{Chemical Engineering Research and Design}
  \bibinfo{volume}{186}, \bibinfo{pages}{34--49}.
\newblock \DOIprefix\doi{https://doi.org/10.1016/j.cherd.2022.07.035}.
\bibitem[{Anderson et~al.(2020)Anderson, Huchette, Ma, Tjandraatmadja and
  Vielma}]{Anderson2020}
\bibinfo{author}{Anderson, R.}, \bibinfo{author}{Huchette, J.},
  \bibinfo{author}{Ma, W.}, \bibinfo{author}{Tjandraatmadja, C.},
  \bibinfo{author}{Vielma, J.P.}, \bibinfo{year}{2020}.
\newblock \bibinfo{title}{Strong mixed-integer programming formulations for
  trained neural networks}.
\newblock \bibinfo{journal}{Mathematical Programming} \bibinfo{volume}{183},
  \bibinfo{pages}{3--39}.
\newblock \DOIprefix\doi{10.1007/s10107-020-01474-5}.
\bibitem[{Bhosekar and Ierapetritou(2018)}]{Bhosekar2018}
\bibinfo{author}{Bhosekar, A.}, \bibinfo{author}{Ierapetritou, M.},
  \bibinfo{year}{2018}.
\newblock \bibinfo{title}{Advances in surrogate based modeling, feasibility
  analysis, and optimization: A review}.
\newblock \bibinfo{journal}{Computers \& Chemical Engineering}
  \bibinfo{volume}{108}, \bibinfo{pages}{250--267}.
\newblock \DOIprefix\doi{https://doi.org/10.1016/j.compchemeng.2017.09.017}.
\bibitem[{Bonzanini et~al.(2021)Bonzanini, Paulson, Makrygiorgos and
  Mesbah}]{Bonzanini2021}
\bibinfo{author}{Bonzanini, A.D.}, \bibinfo{author}{Paulson, J.A.},
  \bibinfo{author}{Makrygiorgos, G.}, \bibinfo{author}{Mesbah, A.},
  \bibinfo{year}{2021}.
\newblock \bibinfo{title}{Fast approximate learning-based multistage nonlinear
  model predictive control using gaussian processes and deep neural networks}.
\newblock \bibinfo{journal}{Computers \& Chemical Engineering}
  \bibinfo{volume}{145}, \bibinfo{pages}{107174}.
\newblock \DOIprefix\doi{https://doi.org/10.1016/j.compchemeng.2020.107174}.
\bibitem[{Boukouvala and Ierapetritou(2012)}]{Boukouvala2012}
\bibinfo{author}{Boukouvala, F.}, \bibinfo{author}{Ierapetritou, M.G.},
  \bibinfo{year}{2012}.
\newblock \bibinfo{title}{Feasibility analysis of black-box processes using an
  adaptive sampling kriging-based method}.
\newblock \bibinfo{journal}{Computers \& Chemical Engineering}
  \bibinfo{volume}{36}, \bibinfo{pages}{358--368}.
\newblock \DOIprefix\doi{https://doi.org/10.1016/j.compchemeng.2011.06.005}.
\bibitem[{Boukouvala and Ierapetritou(2013)}]{Boukouvala2013}
\bibinfo{author}{Boukouvala, F.}, \bibinfo{author}{Ierapetritou, M.G.},
  \bibinfo{year}{2013}.
\newblock \bibinfo{title}{Surrogate-based optimization of expensive flowsheet
  modeling for continuous pharmaceutical manufacturing}.
\newblock \bibinfo{journal}{Journal of Pharmaceutical Innovation}
  \bibinfo{volume}{8}, \bibinfo{pages}{131--145}.
\newblock \DOIprefix\doi{10.1007/s12247-013-9154-1}.
\bibitem[{Bradley et~al.(2022)Bradley, Kim, Kilwein, Blakely, Eydenberg,
  Jalvin, Laird and Boukouvala}]{Bradley2022}
\bibinfo{author}{Bradley, W.}, \bibinfo{author}{Kim, J.},
  \bibinfo{author}{Kilwein, Z.}, \bibinfo{author}{Blakely, L.},
  \bibinfo{author}{Eydenberg, M.}, \bibinfo{author}{Jalvin, J.},
  \bibinfo{author}{Laird, C.}, \bibinfo{author}{Boukouvala, F.},
  \bibinfo{year}{2022}.
\newblock \bibinfo{title}{Perspectives on the integration between
  first-principles and data-driven modeling}.
\newblock \bibinfo{journal}{Computers \& Chemical Engineering}
  \bibinfo{volume}{166}, \bibinfo{pages}{107898}.
\newblock \DOIprefix\doi{https://doi.org/10.1016/j.compchemeng.2022.107898}.
\bibitem[{Cai et~al.(2021)Cai, Wang, Wang, Perdikaris and
  Karniadakis}]{Cai2021}
\bibinfo{author}{Cai, S.}, \bibinfo{author}{Wang, Z.}, \bibinfo{author}{Wang,
  S.}, \bibinfo{author}{Perdikaris, P.}, \bibinfo{author}{Karniadakis, G.E.},
  \bibinfo{year}{2021}.
\newblock \bibinfo{title}{Physics-informed neural networks for heat transfer
  problems}.
\newblock \bibinfo{journal}{Journal of Heat Transfer} \bibinfo{volume}{143}.
\newblock \DOIprefix\doi{10.1115/1.4050542}.
\bibitem[{Ceccon et~al.(2022)Ceccon, Jalving, Haddad, Thebelt, Tsay, Laird and
  Misener}]{Ceccon2022}
\bibinfo{author}{Ceccon, F.}, \bibinfo{author}{Jalving, J.},
  \bibinfo{author}{Haddad, J.}, \bibinfo{author}{Thebelt, A.},
  \bibinfo{author}{Tsay, C.}, \bibinfo{author}{Laird, C.D.},
  \bibinfo{author}{Misener, R.}, \bibinfo{year}{2022}.
\newblock \bibinfo{title}{Omlt: Optimization \& machine learning toolkit}.
\newblock \bibinfo{journal}{J. Mach. Learn. Res.} \bibinfo{volume}{23}.
\bibitem[{Cozad et~al.(2014)Cozad, Sahinidis and Miller}]{Cozad2014}
\bibinfo{author}{Cozad, A.}, \bibinfo{author}{Sahinidis, N.V.},
  \bibinfo{author}{Miller, D.C.}, \bibinfo{year}{2014}.
\newblock \bibinfo{title}{Learning surrogate models for simulation-based
  optimization}.
\newblock \bibinfo{journal}{AIChE Journal} \bibinfo{volume}{60},
  \bibinfo{pages}{2211--2227}.
\newblock \DOIprefix\doi{https://doi.org/10.1002/aic.14418}.
\bibitem[{Dai et~al.(2022)Dai, Mohammadi and Cremaschi}]{Dai2022}
\bibinfo{author}{Dai, W.}, \bibinfo{author}{Mohammadi, S.},
  \bibinfo{author}{Cremaschi, S.}, \bibinfo{year}{2022}.
\newblock \bibinfo{title}{A hybrid modeling framework using dimensional
  analysis for erosion predictions}.
\newblock \bibinfo{journal}{Computers \& Chemical Engineering}
  \bibinfo{volume}{156}, \bibinfo{pages}{107577}.
\newblock \DOIprefix\doi{https://doi.org/10.1016/j.compchemeng.2021.107577}.
\bibitem[{Dener et~al.(2020)Dener, Miller, Churchill, Munson and
  Chang}]{Dener2020}
\bibinfo{author}{Dener, A.}, \bibinfo{author}{Miller, M.A.},
  \bibinfo{author}{Churchill, R.M.}, \bibinfo{author}{Munson, T.S.},
  \bibinfo{author}{Chang, C.S.}, \bibinfo{year}{2020}.
\newblock \bibinfo{title}{Training neural networks under physical constraints
  using a stochastic augmented lagrangian approach}.
\newblock \bibinfo{journal}{ArXiv} \bibinfo{volume}{abs/2009.07330}.
\bibitem[{Dias and Ierapetritou(2020)}]{Dias2020}
\bibinfo{author}{Dias, L.S.}, \bibinfo{author}{Ierapetritou, M.G.},
  \bibinfo{year}{2020}.
\newblock \bibinfo{title}{Integration of planning, scheduling and control
  problems using data-driven feasibility analysis and surrogate models}.
\newblock \bibinfo{journal}{Computers \& Chemical Engineering}
  \bibinfo{volume}{134}, \bibinfo{pages}{106714}.
\newblock \DOIprefix\doi{https://doi.org/10.1016/j.compchemeng.2019.106714}.
\bibitem[{Fischetti and Jo(2018)}]{Fischetti2018}
\bibinfo{author}{Fischetti, M.}, \bibinfo{author}{Jo, J.},
  \bibinfo{year}{2018}.
\newblock \bibinfo{title}{Deep neural networks and mixed integer linear
  optimization}.
\newblock \bibinfo{journal}{Constraints} \bibinfo{volume}{23},
  \bibinfo{pages}{296--309}.
\newblock \DOIprefix\doi{10.1007/s10601-018-9285-6}.
\bibitem[{Goldstein et~al.(2022)Goldstein, Heyer, Jakobs, Schultz and
  Biegler}]{Goldstein2022}
\bibinfo{author}{Goldstein, D.}, \bibinfo{author}{Heyer, M.},
  \bibinfo{author}{Jakobs, D.}, \bibinfo{author}{Schultz, E.S.},
  \bibinfo{author}{Biegler, L.T.}, \bibinfo{year}{2022}.
\newblock \bibinfo{title}{Multilevel surrogate modeling of an amine scrubbing
  process for co2capture}.
\newblock \bibinfo{journal}{AIChE Journal} \bibinfo{volume}{68},
  \bibinfo{pages}{e17705}.
\newblock \DOIprefix\doi{https://doi.org/10.1002/aic.17705}.
\bibitem[{Grimstad and Andersson(2019)}]{Grimstad2019}
\bibinfo{author}{Grimstad, B.}, \bibinfo{author}{Andersson, H.},
  \bibinfo{year}{2019}.
\newblock \bibinfo{title}{Relu networks as surrogate models in mixed-integer
  linear programs}.
\newblock \bibinfo{journal}{Computers \& Chemical Engineering}
  \bibinfo{volume}{131}, \bibinfo{pages}{106580}.
\newblock \DOIprefix\doi{https://doi.org/10.1016/j.compchemeng.2019.106580}.
\bibitem[{Henao and Maravelias(2011)}]{Henao2011}
\bibinfo{author}{Henao, C.A.}, \bibinfo{author}{Maravelias, C.T.},
  \bibinfo{year}{2011}.
\newblock \bibinfo{title}{Surrogate-based superstructure optimization
  framework}.
\newblock \bibinfo{journal}{AIChE Journal} \bibinfo{volume}{57},
  \bibinfo{pages}{1216--1232}.
\newblock \DOIprefix\doi{https://doi.org/10.1002/aic.12341}.
\bibitem[{Hinton et~al.(2012)Hinton, Deng, Yu, Dahl, rahman Mohamed, Jaitly,
  Senior, Vanhoucke, Nguyen, Sainath and Kingsbury}]{Hinton2012}
\bibinfo{author}{Hinton, G.}, \bibinfo{author}{Deng, L.}, \bibinfo{author}{Yu,
  D.}, \bibinfo{author}{Dahl, G.E.}, \bibinfo{author}{rahman Mohamed, A.},
  \bibinfo{author}{Jaitly, N.}, \bibinfo{author}{Senior, A.},
  \bibinfo{author}{Vanhoucke, V.}, \bibinfo{author}{Nguyen, P.},
  \bibinfo{author}{Sainath, T.N.}, \bibinfo{author}{Kingsbury, B.},
  \bibinfo{year}{2012}.
\newblock \bibinfo{title}{Deep neural networks for acoustic modeling in speech
  recognition: The shared views of four research groups}.
\newblock \bibinfo{journal}{IEEE Signal Processing Magazine}
  \bibinfo{volume}{29}, \bibinfo{pages}{82--97}.
\newblock \DOIprefix\doi{10.1109/MSP.2012.2205597}.
\bibitem[{Jidling et~al.(2017)Jidling, Wahlstr\"{o}m, Wills and
  Sch\"{o}n}]{Jidling2017}
\bibinfo{author}{Jidling, C.}, \bibinfo{author}{Wahlstr\"{o}m, N.},
  \bibinfo{author}{Wills, A.}, \bibinfo{author}{Sch\"{o}n, T.B.},
  \bibinfo{year}{2017}.
\newblock \bibinfo{title}{Linearly constrained gaussian processes}, in:
  \bibinfo{booktitle}{Advances in Neural Information Processing Systems},
  \bibinfo{publisher}{Curran Associates, Inc.}
\bibitem[{Jumper et~al.(2021)Jumper, Evans, Pritzel, Green, Figurnov,
  Ronneberger, Tunyasuvunakool, Bates, Žídek, Potapenko, Bridgland, Meyer,
  Kohl, Ballard, Cowie, Romera-Paredes, Nikolov, Jain, Adler, Back, Petersen,
  Reiman, Clancy, Zielinski, Steinegger, Pacholska, Berghammer, Bodenstein,
  Silver, Vinyals, Senior, Kavukcuoglu, Kohli and Hassabis}]{Jumper2021}
\bibinfo{author}{Jumper, J.}, \bibinfo{author}{Evans, R.},
  \bibinfo{author}{Pritzel, A.}, \bibinfo{author}{Green, T.},
  \bibinfo{author}{Figurnov, M.}, \bibinfo{author}{Ronneberger, O.},
  \bibinfo{author}{Tunyasuvunakool, K.}, \bibinfo{author}{Bates, R.},
  \bibinfo{author}{Žídek, A.}, \bibinfo{author}{Potapenko, A.},
  \bibinfo{author}{Bridgland, A.}, \bibinfo{author}{Meyer, C.},
  \bibinfo{author}{Kohl, S.A.A.}, \bibinfo{author}{Ballard, A.J.},
  \bibinfo{author}{Cowie, A.}, \bibinfo{author}{Romera-Paredes, B.},
  \bibinfo{author}{Nikolov, S.}, \bibinfo{author}{Jain, R.},
  \bibinfo{author}{Adler, J.}, \bibinfo{author}{Back, T.},
  \bibinfo{author}{Petersen, S.}, \bibinfo{author}{Reiman, D.},
  \bibinfo{author}{Clancy, E.}, \bibinfo{author}{Zielinski, M.},
  \bibinfo{author}{Steinegger, M.}, \bibinfo{author}{Pacholska, M.},
  \bibinfo{author}{Berghammer, T.}, \bibinfo{author}{Bodenstein, S.},
  \bibinfo{author}{Silver, D.}, \bibinfo{author}{Vinyals, O.},
  \bibinfo{author}{Senior, A.W.}, \bibinfo{author}{Kavukcuoglu, K.},
  \bibinfo{author}{Kohli, P.}, \bibinfo{author}{Hassabis, D.},
  \bibinfo{year}{2021}.
\newblock \bibinfo{title}{Highly accurate protein structure prediction with
  alphafold}.
\newblock \bibinfo{journal}{Nature} \bibinfo{volume}{596},
  \bibinfo{pages}{583--589}.
\newblock \DOIprefix\doi{10.1038/s41586-021-03819-2}.
\bibitem[{Karniadakis et~al.(2021)Karniadakis, Kevrekidis, Lu, Perdikaris, Wang
  and Yang}]{Karniadakis2021}
\bibinfo{author}{Karniadakis, G.E.}, \bibinfo{author}{Kevrekidis, I.G.},
  \bibinfo{author}{Lu, L.}, \bibinfo{author}{Perdikaris, P.},
  \bibinfo{author}{Wang, S.}, \bibinfo{author}{Yang, L.}, \bibinfo{year}{2021}.
\newblock \bibinfo{title}{Physics-informed machine learning}.
\newblock \DOIprefix\doi{10.1038/s42254-021-00314-5}.
\bibitem[{Kim et~al.(2023)Kim, Luettgen, Paynabar and Boukouvala}]{Kim2023}
\bibinfo{author}{Kim, J.}, \bibinfo{author}{Luettgen, C.},
  \bibinfo{author}{Paynabar, K.}, \bibinfo{author}{Boukouvala, F.},
  \bibinfo{year}{2023}.
\newblock \bibinfo{title}{Physics-based penalization for hyperparameter
  estimation in gaussian process regression}.
\newblock \bibinfo{journal}{Computers \& Chemical Engineering}
  \bibinfo{volume}{178}, \bibinfo{pages}{108320}.
\newblock \DOIprefix\doi{https://doi.org/10.1016/j.compchemeng.2023.108320}.
\bibitem[{Kim and Boukouvala(2020)}]{Kim2020}
\bibinfo{author}{Kim, S.H.}, \bibinfo{author}{Boukouvala, F.},
  \bibinfo{year}{2020}.
\newblock \bibinfo{title}{Surrogate-based optimization for mixed-integer
  nonlinear problems}.
\newblock \bibinfo{journal}{Computers \& Chemical Engineering}
  \bibinfo{volume}{140}, \bibinfo{pages}{106847}.
\newblock \DOIprefix\doi{https://doi.org/10.1016/j.compchemeng.2020.106847}.
\bibitem[{Krizhevsky et~al.(2012)Krizhevsky, Sutskever and
  Hinton}]{Krizhevsky2012}
\bibinfo{author}{Krizhevsky, A.}, \bibinfo{author}{Sutskever, I.},
  \bibinfo{author}{Hinton, G.E.}, \bibinfo{year}{2012}.
\newblock \bibinfo{title}{Imagenet classification with deep convolutional
  neural networks}, in: \bibinfo{booktitle}{Advances in Neural Information
  Processing Systems}, \bibinfo{publisher}{Curran Associates, Inc.}
\bibitem[{LeCun et~al.(2015)LeCun, Bengio and Hinton}]{LeCun2015}
\bibinfo{author}{LeCun, Y.}, \bibinfo{author}{Bengio, Y.},
  \bibinfo{author}{Hinton, G.}, \bibinfo{year}{2015}.
\newblock \bibinfo{title}{Deep learning}.
\newblock \bibinfo{journal}{Nature} \bibinfo{volume}{521},
  \bibinfo{pages}{436--444}.
\newblock \DOIprefix\doi{10.1038/nature14539}.
\bibitem[{Lu et~al.(2021)Lu, Pestourie, Yao, Wang, Verdugo and
  Johnson}]{Lu2021}
\bibinfo{author}{Lu, L.}, \bibinfo{author}{Pestourie, R.},
  \bibinfo{author}{Yao, W.}, \bibinfo{author}{Wang, Z.},
  \bibinfo{author}{Verdugo, F.}, \bibinfo{author}{Johnson, S.G.},
  \bibinfo{year}{2021}.
\newblock \bibinfo{title}{Physics-informed neural networks with hard
  constraints for inverse design}.
\newblock \bibinfo{journal}{SIAM Journal on Scientific Computing}
  \bibinfo{volume}{43}, \bibinfo{pages}{B1105--B1132}.
\newblock \DOIprefix\doi{10.1137/21M1397908}.
\bibitem[{Ma et~al.(2022a)Ma, Sahinidis, Amaran, Bindlish, Bury, Griffith and
  Rajagopalan}]{Ma2022a}
\bibinfo{author}{Ma, K.}, \bibinfo{author}{Sahinidis, N.V.},
  \bibinfo{author}{Amaran, S.}, \bibinfo{author}{Bindlish, R.},
  \bibinfo{author}{Bury, S.J.}, \bibinfo{author}{Griffith, D.},
  \bibinfo{author}{Rajagopalan, S.}, \bibinfo{year}{2022}a.
\newblock \bibinfo{title}{Data-driven strategies for optimization of integrated
  chemical plants}.
\newblock \bibinfo{journal}{Computers \& Chemical Engineering}
  \bibinfo{volume}{166}, \bibinfo{pages}{107961}.
\newblock \DOIprefix\doi{https://doi.org/10.1016/j.compchemeng.2022.107961}.
\bibitem[{Ma et~al.(2022b)Ma, Sahinidis, Bindlish, Bury, Haghpanah and
  Rajagopalan}]{Ma2022b}
\bibinfo{author}{Ma, K.}, \bibinfo{author}{Sahinidis, N.V.},
  \bibinfo{author}{Bindlish, R.}, \bibinfo{author}{Bury, S.J.},
  \bibinfo{author}{Haghpanah, R.}, \bibinfo{author}{Rajagopalan, S.},
  \bibinfo{year}{2022}b.
\newblock \bibinfo{title}{Data-driven strategies for extractive distillation
  unit optimization}.
\newblock \bibinfo{journal}{Computers \& Chemical Engineering}
  \bibinfo{volume}{167}, \bibinfo{pages}{107970}.
\bibitem[{Mehrian et~al.(2018)Mehrian, Guyot, Papantoniou, Olofsson, Sonnaert,
  Misener and Geris}]{Mehrian2018}
\bibinfo{author}{Mehrian, M.}, \bibinfo{author}{Guyot, Y.},
  \bibinfo{author}{Papantoniou, I.}, \bibinfo{author}{Olofsson, S.},
  \bibinfo{author}{Sonnaert, M.}, \bibinfo{author}{Misener, R.},
  \bibinfo{author}{Geris, L.}, \bibinfo{year}{2018}.
\newblock \bibinfo{title}{Maximizing neotissue growth kinetics in a perfusion
  bioreactor: An in silico strategy using model reduction and bayesian
  optimization}.
\newblock \bibinfo{journal}{Biotechnology and Bioengineering}
  \bibinfo{volume}{115}, \bibinfo{pages}{617--629}.
\newblock \DOIprefix\doi{https://doi.org/10.1002/bit.26500}.
\bibitem[{Misener and Biegler(2023)}]{Misener2023}
\bibinfo{author}{Misener, R.}, \bibinfo{author}{Biegler, L.},
  \bibinfo{year}{2023}.
\newblock \bibinfo{title}{Formulating data-driven surrogate models for process
  optimization}.
\newblock \bibinfo{journal}{Computers \& Chemical Engineering}
  \bibinfo{volume}{179}, \bibinfo{pages}{108411}.
\newblock \DOIprefix\doi{https://doi.org/10.1016/j.compchemeng.2023.108411}.
\bibitem[{Mohammadi et~al.(2022)Mohammadi, Williams and
  Cremaschi}]{Mohammadi2022}
\bibinfo{author}{Mohammadi, S.}, \bibinfo{author}{Williams, B.},
  \bibinfo{author}{Cremaschi, S.}, \bibinfo{year}{2022}.
\newblock \bibinfo{title}{Surrogate Modeling and Surrogate-Based Optimization
  with Stochastic Simulations}. \bibinfo{publisher}{Elsevier}.
  volume~\bibinfo{volume}{49}.
\newblock pp. \bibinfo{pages}{31--40}.
\newblock \DOIprefix\doi{https://doi.org/10.1016/B978-0-323-85159-6.50005-1}.
\bibitem[{Na et~al.(2021)Na, Bak and Sahinidis}]{Na2021}
\bibinfo{author}{Na, J.}, \bibinfo{author}{Bak, J.H.},
  \bibinfo{author}{Sahinidis, N.V.}, \bibinfo{year}{2021}.
\newblock \bibinfo{title}{Efficient bayesian inference using adversarial
  machine learning and low-complexity surrogate models}.
\newblock \bibinfo{journal}{Computers \& Chemical Engineering}
  \bibinfo{volume}{151}, \bibinfo{pages}{107322}.
\newblock \DOIprefix\doi{https://doi.org/10.1016/j.compchemeng.2021.107322}.
\bibitem[{Nodozi et~al.(2023)Nodozi, O’Leary, Mesbah and Halder}]{Nodozi2023}
\bibinfo{author}{Nodozi, I.}, \bibinfo{author}{O’Leary, J.},
  \bibinfo{author}{Mesbah, A.}, \bibinfo{author}{Halder, A.},
  \bibinfo{year}{2023}.
\newblock \bibinfo{title}{A physics-informed deep learning approach for minimum
  effort stochastic control of colloidal self-assembly}, in:
  \bibinfo{booktitle}{2023 American Control Conference (ACC)},
  \bibinfo{publisher}{IEEE}. pp. \bibinfo{pages}{609--615}.
\bibitem[{O'Leary et~al.(2022)O'Leary, Paulson and Mesbah}]{OLeary2022}
\bibinfo{author}{O'Leary, J.}, \bibinfo{author}{Paulson, J.A.},
  \bibinfo{author}{Mesbah, A.}, \bibinfo{year}{2022}.
\newblock \bibinfo{title}{Stochastic physics-informed neural ordinary
  differential equations}.
\newblock \bibinfo{journal}{Journal of Computational Physics}
  \bibinfo{volume}{468}, \bibinfo{pages}{111466}.
\newblock \DOIprefix\doi{https://doi.org/10.1016/j.jcp.2022.111466}.
\bibitem[{Olofsson et~al.(2018)Olofsson, Deisenroth and Misener}]{Olofsson2018}
\bibinfo{author}{Olofsson, S.}, \bibinfo{author}{Deisenroth, M.P.},
  \bibinfo{author}{Misener, R.}, \bibinfo{year}{2018}.
\newblock \bibinfo{title}{Design of Experiments for Model Discrimination using
  Gaussian Process Surrogate Models}. \bibinfo{publisher}{Elsevier}.
  volume~\bibinfo{volume}{44}.
\newblock pp. \bibinfo{pages}{847--852}.
\newblock \DOIprefix\doi{https://doi.org/10.1016/B978-0-444-64241-7.50136-1}.
\bibitem[{Paszke et~al.(2019)Paszke, Gross, Massa, Lerer, Bradbury, Chanan,
  Killeen, Lin, Gimelshein, Antiga, Desmaison, Kopf, Yang, DeVito, Raison,
  Tejani, Chilamkurthy, Steiner, Fang, Bai and Chintala}]{Paszke2019}
\bibinfo{author}{Paszke, A.}, \bibinfo{author}{Gross, S.},
  \bibinfo{author}{Massa, F.}, \bibinfo{author}{Lerer, A.},
  \bibinfo{author}{Bradbury, J.}, \bibinfo{author}{Chanan, G.},
  \bibinfo{author}{Killeen, T.}, \bibinfo{author}{Lin, Z.},
  \bibinfo{author}{Gimelshein, N.}, \bibinfo{author}{Antiga, L.},
  \bibinfo{author}{Desmaison, A.}, \bibinfo{author}{Kopf, A.},
  \bibinfo{author}{Yang, E.}, \bibinfo{author}{DeVito, Z.},
  \bibinfo{author}{Raison, M.}, \bibinfo{author}{Tejani, A.},
  \bibinfo{author}{Chilamkurthy, S.}, \bibinfo{author}{Steiner, B.},
  \bibinfo{author}{Fang, L.}, \bibinfo{author}{Bai, J.},
  \bibinfo{author}{Chintala, S.}, \bibinfo{year}{2019}.
\newblock \bibinfo{title}{Pytorch: An imperative style, high-performance deep
  learning library}, in: \bibinfo{booktitle}{Advances in Neural Information
  Processing Systems}, \bibinfo{publisher}{Curran Associates, Inc.}
\bibitem[{Paulson et~al.(2021)Paulson, Shao and Mesbah}]{Paulson2021}
\bibinfo{author}{Paulson, J.A.}, \bibinfo{author}{Shao, K.},
  \bibinfo{author}{Mesbah, A.}, \bibinfo{year}{2021}.
\newblock \bibinfo{title}{Probabilistically robust bayesian optimization for
  data-driven design of arbitrary controllers with gaussian process emulators},
  in: \bibinfo{booktitle}{2021 60th IEEE Conference on Decision and Control
  (CDC)}, pp. \bibinfo{pages}{3633--3639}.
\newblock \DOIprefix\doi{10.1109/CDC45484.2021.9683046}.
\bibitem[{Quirante et~al.(2015)Quirante, Javaloyes, Ruiz-Femenia and
  Caballero}]{Quirante2015}
\bibinfo{author}{Quirante, N.}, \bibinfo{author}{Javaloyes, J.},
  \bibinfo{author}{Ruiz-Femenia, R.}, \bibinfo{author}{Caballero, J.A.},
  \bibinfo{year}{2015}.
\newblock \bibinfo{title}{Optimization of Chemical Processes Using Surrogate
  Models Based on a Kriging Interpolation}. \bibinfo{publisher}{Elsevier}.
  volume~\bibinfo{volume}{37}.
\newblock pp. \bibinfo{pages}{179--184}.
\newblock \DOIprefix\doi{https://doi.org/10.1016/B978-0-444-63578-5.50025-6}.
\bibitem[{Raissi et~al.(2019)Raissi, Perdikaris and Karniadakis}]{Raissi2019}
\bibinfo{author}{Raissi, M.}, \bibinfo{author}{Perdikaris, P.},
  \bibinfo{author}{Karniadakis, G.}, \bibinfo{year}{2019}.
\newblock \bibinfo{title}{Physics-informed neural networks: A deep learning
  framework for solving forward and inverse problems involving nonlinear
  partial differential equations}.
\newblock \bibinfo{journal}{Journal of Computational Physics}
  \bibinfo{volume}{378}, \bibinfo{pages}{686--707}.
\newblock \DOIprefix\doi{10.1016/j.jcp.2018.10.045}.
\bibitem[{Raissi et~al.(2020)Raissi, Yazdani and Karniadakis}]{Raissi2020}
\bibinfo{author}{Raissi, M.}, \bibinfo{author}{Yazdani, A.},
  \bibinfo{author}{Karniadakis, G.E.}, \bibinfo{year}{2020}.
\newblock \bibinfo{title}{Hidden fluid mechanics: Learning velocity and
  pressure fields from flow visualizations}.
\newblock \bibinfo{journal}{Science} \bibinfo{volume}{367},
  \bibinfo{pages}{1026--1030}.
\newblock \DOIprefix\doi{10.1126/science.aaw4741}.
\bibitem[{Rasmussen and Williams(2005)}]{Rasmussen2005}
\bibinfo{author}{Rasmussen, C.E.}, \bibinfo{author}{Williams, C.K.I.},
  \bibinfo{year}{2005}.
\newblock \bibinfo{title}{{Gaussian Processes for Machine Learning}}.
\newblock \bibinfo{publisher}{The MIT Press}.
\newblock \DOIprefix\doi{10.7551/mitpress/3206.001.0001}.
\bibitem[{Schweidtmann et~al.(2021)Schweidtmann, Bongartz, Grothe, Kerkenhoff,
  Lin, Najman and Mitsos}]{Schweidtmann2021}
\bibinfo{author}{Schweidtmann, A.M.}, \bibinfo{author}{Bongartz, D.},
  \bibinfo{author}{Grothe, D.}, \bibinfo{author}{Kerkenhoff, T.},
  \bibinfo{author}{Lin, X.}, \bibinfo{author}{Najman, J.},
  \bibinfo{author}{Mitsos, A.}, \bibinfo{year}{2021}.
\newblock \bibinfo{title}{Deterministic global optimization with gaussian
  processes embedded}.
\newblock \bibinfo{journal}{Mathematical Programming Computation}
  \bibinfo{volume}{13}, \bibinfo{pages}{553--581}.
\newblock \DOIprefix\doi{10.1007/s12532-021-00204-y}.
\bibitem[{Schweidtmann and Mitsos(2019)}]{Schweidtmann2019}
\bibinfo{author}{Schweidtmann, A.M.}, \bibinfo{author}{Mitsos, A.},
  \bibinfo{year}{2019}.
\newblock \bibinfo{title}{Deterministic global optimization with artificial
  neural networks embedded}.
\newblock \bibinfo{journal}{Journal of Optimization Theory and Applications}
  \bibinfo{volume}{180}, \bibinfo{pages}{925--948}.
\newblock \DOIprefix\doi{10.1007/s10957-018-1396-0}.
\bibitem[{Sutskever et~al.(2014)Sutskever, Vinyals and Le}]{Sutskever2014}
\bibinfo{author}{Sutskever, I.}, \bibinfo{author}{Vinyals, O.},
  \bibinfo{author}{Le, Q.V.}, \bibinfo{year}{2014}.
\newblock \bibinfo{title}{Sequence to sequence learning with neural networks},
  in: \bibinfo{booktitle}{Advances in Neural Information Processing Systems},
  \bibinfo{publisher}{Curran Associates, Inc.}
\bibitem[{Szegedy et~al.(2015)Szegedy, Liu, Jia, Sermanet, Reed, Anguelov,
  Erhan, Vanhoucke and Rabinovich}]{Szegedy2014}
\bibinfo{author}{Szegedy, C.}, \bibinfo{author}{Liu, W.}, \bibinfo{author}{Jia,
  Y.}, \bibinfo{author}{Sermanet, P.}, \bibinfo{author}{Reed, S.},
  \bibinfo{author}{Anguelov, D.}, \bibinfo{author}{Erhan, D.},
  \bibinfo{author}{Vanhoucke, V.}, \bibinfo{author}{Rabinovich, A.},
  \bibinfo{year}{2015}.
\newblock \bibinfo{title}{Going deeper with convolutions}, in:
  \bibinfo{booktitle}{2015 IEEE Conference on Computer Vision and Pattern
  Recognition (CVPR)}, \bibinfo{publisher}{IEEE Computer Society},
  \bibinfo{address}{Los Alamitos, CA, USA}. pp. \bibinfo{pages}{1--9}.
\newblock \DOIprefix\doi{10.1109/CVPR.2015.7298594}.
\bibitem[{Tsay et~al.(2021)Tsay, Kronqvist, Thebelt and Misener}]{Tsay2021}
\bibinfo{author}{Tsay, C.}, \bibinfo{author}{Kronqvist, J.},
  \bibinfo{author}{Thebelt, A.}, \bibinfo{author}{Misener, R.},
  \bibinfo{year}{2021}.
\newblock \bibinfo{title}{Partition-based formulations for mixed-integer
  optimization of trained relu neural networks}, in:
  \bibinfo{booktitle}{Advances in Neural Information Processing Systems},
  \bibinfo{publisher}{Curran Associates, Inc.}. pp.
  \bibinfo{pages}{3068--3080}.
\bibitem[{Wiebe et~al.(2022)Wiebe, Cecílio, Dunlop and Misener}]{Wiebe2022}
\bibinfo{author}{Wiebe, J.}, \bibinfo{author}{Cecílio, I.},
  \bibinfo{author}{Dunlop, J.}, \bibinfo{author}{Misener, R.},
  \bibinfo{year}{2022}.
\newblock \bibinfo{title}{A robust approach to warped gaussian
  process-constrained optimization}.
\newblock \bibinfo{journal}{Mathematical Programming} \bibinfo{volume}{196},
  \bibinfo{pages}{805--839}.
\newblock \DOIprefix\doi{10.1007/s10107-021-01762-8}.
\bibitem[{Williams and Cremaschi(2021)}]{Williams2021}
\bibinfo{author}{Williams, B.}, \bibinfo{author}{Cremaschi, S.},
  \bibinfo{year}{2021}.
\newblock \bibinfo{title}{Selection of surrogate modeling techniques for
  surface approximation and surrogate-based optimization}.
\newblock \bibinfo{journal}{Chemical Engineering Research and Design}
  \bibinfo{volume}{170}, \bibinfo{pages}{76--89}.
\newblock \DOIprefix\doi{https://doi.org/10.1016/j.cherd.2021.03.028}.
\bibitem[{Wilson and Sahinidis(2017)}]{Wilson2017}
\bibinfo{author}{Wilson, Z.T.}, \bibinfo{author}{Sahinidis, N.V.},
  \bibinfo{year}{2017}.
\newblock \bibinfo{title}{The alamo approach to machine learning}.
\newblock \bibinfo{journal}{Computers \& Chemical Engineering}
  \bibinfo{volume}{106}, \bibinfo{pages}{785--795}.
\newblock \DOIprefix\doi{https://doi.org/10.1016/j.compchemeng.2017.02.010}.
\bibitem[{Wilson and Sahinidis(2019)}]{Wilson2019}
\bibinfo{author}{Wilson, Z.T.}, \bibinfo{author}{Sahinidis, N.V.},
  \bibinfo{year}{2019}.
\newblock \bibinfo{title}{Automated learning of chemical reaction networks}.
\newblock \bibinfo{journal}{Computers \& Chemical Engineering}
  \bibinfo{volume}{127}, \bibinfo{pages}{88--98}.
\newblock \DOIprefix\doi{https://doi.org/10.1016/j.compchemeng.2019.05.020}.
\bibitem[{Zheng et~al.(2023)Zheng, Hu, Wang and Wu}]{Zheng2023b}
\bibinfo{author}{Zheng, Y.}, \bibinfo{author}{Hu, C.}, \bibinfo{author}{Wang,
  X.}, \bibinfo{author}{Wu, Z.}, \bibinfo{year}{2023}.
\newblock \bibinfo{title}{Physics-informed recurrent neural network modeling
  for predictive control of nonlinear processes}.
\newblock \bibinfo{journal}{Journal of Process Control} \bibinfo{volume}{128},
  \bibinfo{pages}{103005}.
\newblock \DOIprefix\doi{https://doi.org/10.1016/j.jprocont.2023.103005}.
\bibitem[{Zheng and Wu(2023)}]{Zheng2023a}
\bibinfo{author}{Zheng, Y.}, \bibinfo{author}{Wu, Z.}, \bibinfo{year}{2023}.
\newblock \bibinfo{title}{Physics-informed online machine learning and
  predictive control of nonlinear processes with parameter uncertainty}.
\newblock \bibinfo{journal}{Industrial \& Engineering Chemistry Research}
  \bibinfo{volume}{62}, \bibinfo{pages}{2804--2818}.
\newblock \DOIprefix\doi{10.1021/acs.iecr.2c03691}.

\end{thebibliography}





\end{document}